\documentclass[sigconf]{acmart}
\AtBeginDocument{%
  \providecommand\BibTeX{{%
    \normalfont B\kern-0.5em{\scshape i\kern-0.25em b}\kern-0.8em\TeX}}}

\usepackage{multirow}
\usepackage{multicol}
\usepackage{booktabs}
\usepackage{subfigure}
\usepackage{subcaption}
\usepackage{hhline}
\usepackage{color}
\usepackage{pifont}
\usepackage[ruled, vlined]{algorithm2e}
\usepackage{algpseudocode}
\usepackage{colortbl}
\usepackage{enumitem}
\usepackage{pifont}
\newtheorem{theorem}{Theorem}
\usepackage{hyperref}

\hypersetup{
    colorlinks=true,
    linkcolor=blue,
    filecolor=blue,      
    urlcolor=blue,
    citecolor=cyan,
}

\newcommand{\figref}[1]{Figure~\ref{#1}}
\newcommand{\tabref}[1]{Table~\ref{#1}}
\newcommand{\equref}[1]{Eq.~\ref{#1}}
\newcommand{\aloref}[1]{Algorithm~\ref{#1}}
\newcommand{\theref}[1]{Theorem~\ref{#1}}

\newcommand{\best}{\cellcolor[HTML]{DED0B6}}
\newcommand{\second}{\cellcolor[HTML]{FAEED1}}
\definecolor{color1}{RGB}{180, 96, 96} 
\definecolor{color2}{RGB}{33, 70, 199} 

\begin{document}
\title{Rethinking Fair Graph Neural Networks from Re-balancing}

\makeatletter
\def\authornotetext#1{
	\g@addto@macro\@authornotes{%
	\stepcounter{footnote}\footnotetext{#1}}%
}
\makeatother

\copyrightyear{2024}
\acmYear{2024}
\setcopyright{acmlicensed}\acmConference[KDD '24]{Proceedings of the 30th ACM SIGKDD Conference on Knowledge Discovery and Data Mining}{August 25--29, 2024}{Barcelona, Spain}
\acmBooktitle{Proceedings of the 30th ACM SIGKDD Conference on Knowledge Discovery and Data Mining (KDD '24), August 25--29, 2024, Barcelona, Spain}
\acmPrice{}
\acmDOI{10.1145/3637528.3671826}
\acmISBN{979-8-4007-0490-1/24/08}

\author{Zhixun Li}
\affiliation{%
  \institution{The Chinese University of Hong Kong}
    \country{Hong Kong SAR, China}
  }
\email{zxli@se.cuhk.edu.hk}

\author{Yushun Dong}
\affiliation{%
  \institution{University of Virginia}
    \country{Charlottesville, USA}
  }
\email{yd6eb@virginia.edu}

\author{Qiang Liu}
\affiliation{%
  \institution{Institute of Automation, Chinese Academy of Sciences}
    \country{Beijing, China}
  }
\email{qiang.liu@nlpr.ia.ac.cn}

\author{Jeffrey Xu Yu}
\authornote{Corresponding author}
\affiliation{%
  \institution{The Chinese University of Hong Kong}
    \country{Hong Kong SAR, China}
  }
\email{yu@se.cuhk.edu.hk}

\begin{abstract}

Driven by the powerful representation ability of Graph Neural Networks (GNNs), plentiful GNN models have been widely deployed in many real-world applications. Nevertheless, due to distribution disparities between different demographic groups, fairness in high-stake decision-making systems is receiving increasing attention. Although lots of recent works devoted to improving the fairness of GNNs and achieved considerable success, they all require significant architectural changes or additional loss functions requiring more hyper-parameter tuning. Surprisingly, we find that simple re-balancing methods can easily match or surpass existing fair GNN methods. We claim that the imbalance across different demographic groups is a significant source of unfairness, resulting in imbalanced contributions from each group to the parameters updating. However, these simple re-balancing methods have their own shortcomings during training. In this paper, we propose FairGB, \underline{Fair} \underline{G}raph Neural Network via re-\underline{B}alancing, which mitigates the unfairness of GNNs by group balancing. Technically, FairGB consists of two modules: counterfactual node mixup and contribution alignment loss. Firstly, we select counterfactual pairs across inter-domain and inter-class, and interpolate the ego-networks to generate new samples. Guided by analysis, we can reveal the debiasing mechanism of our model by the causal view and prove that our strategy can make sensitive attributes statistically independent from target labels. Secondly, we reweigh the contribution of each group according to gradients. By combining these two modules, they can mutually promote each other. Experimental results on benchmark datasets show that our method can achieve state-of-the-art results concerning both utility and fairness metrics. Code is available at \url{https://github.com/ZhixunLEE/FairGB}.

\end{abstract}



\maketitle
\section{Introduction}

\begin{figure}[t]
\centering
\includegraphics[width=\linewidth]{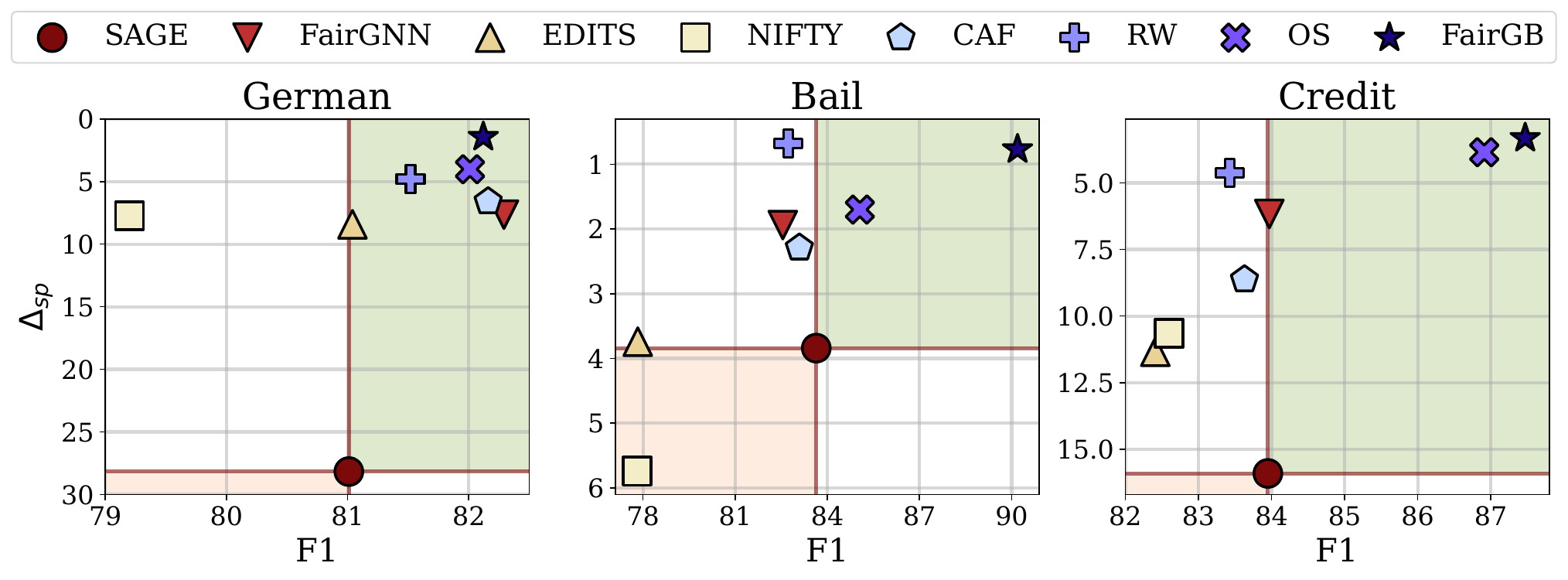}
\caption{The F1-$\Delta_{sp}$ trade-off on German, Bail, and Credit datasets. "RW" denotes re-weighting and "OS" denotes over-sampling. The light green region indicates the model outperforms the vanilla model in both utility and fairness, while the light red region represents the opposite.}
\label{fig:compare}
\vspace{-0.3cm}
\end{figure}

With the rapid development of deep learning, Graph Neural Networks (GNNs) have been widely used in dealing with non-Euclidean graph-structured data and gained a deeper understanding to help us perform predictive tasks, such as molecular property prediction and etc \cite{chen2024uncovering, wieder2020compact, zhang2024graphlta, li2024zerog, zhang2024graph, wang2024heterophilic}. However, potential bias in datasets will lead the neural networks to favor the privileged groups, for example, if the historical data reflects a bias towards loans for individuals from higher-income communities, the model may learn to associate high income with loan approval, leading to the unjustified denial of loans for applicants from lower-income communities even if they are creditworthy. Therefore, fairness in high-stake automatic decision-making systems, \emph{e.g.}, medication recommendation \cite{shang2019gamenet}, fraud detection \cite{li2022devil, neo2024towards}, and credit risk prediction \cite{yeh2009comparisons} have been receiving increasing attention in recent years.

Close to the heels of the wide application of GNNs, there are various fair GNN approaches have been proposed to improve fairness without unduly compromising utility performance \cite{dai2021say, dong2022edits, wang2022improving, agarwal2021towards, ma2022learning,dong2021individual,dong2023fairness,dong2023reliant,zhang2023adversarial}. The two most typical strategies of fair GNNs are imposing the fairness consideration as a regularization term during optimization \cite{dai2021say, wang2022improving, guo2023towards,song2022guide} or modifying the original training data with the assumption that fair data would result in a fair model \cite{dong2022edits, rahman2019fairwalk}. However, previous works focus on eliminating information about sensitive attributes but overlook the fact that due to attribute imbalance, underprivileged groups with fewer training samples are underrepresented compared to the privileged groups with more training samples. Inspired by recent work in maximizing worst-group-accuracy \cite{idrissi2022simple}, we surprisingly observe that simple group-wise re-balancing methods (\emph{e.g.}, re-sampling and re-weighting) can easily achieve competitive or even superior results compared to existing state-of-the-art fair GNN models with no additional hyper-parameters (as shown in \figref{fig:compare}). Data re-balancing methods are popular within the class imbalance literature, but we focus on balancing each group in this paper, which we refer to as group-wise re-balancing. However, these simple re-balancing methods have their own shortcomings on the imbalanced graph-structured datasets. (1) \emph{\textbf{Re-sampling:}} Down- and over-sampling will mitigate imbalance distribution by dropping or duplicating training samples. However down-sampling will lose a lot of beneficial information from the training set, jeopardizing the utility of the model. While over-sampling simply repeats minority samples will cause an over-fitting problem and hard to connect adjacent edges of newly generated nodes because of non-iid characteristics of graphs. (2) \emph{\textbf{Re-weighting:}} re-weighting methods apply penalties according to the quantity of groups and assign large weights to minor groups. However, up-weighing nodes in minor groups may also result in the over-fitting problem, and inevitably increase the false positive cases for major nodes \cite{song2022tam}.

Several works attempt to alleviate the class imbalance in the node classification scenario \cite{park2022graphens, song2022tam, liu2023survey, li2023graphsha}. For instance, GraphENS \cite{park2022graphens} injects the whole synthesized ego-networks for minor class. TAM \cite{song2022tam} designs a node-wise logit adjustment method, which adaptively adjusts the margin accordingly based on local topology. However they did not take into account sub-groups within the classes, and to the best of our knowledge, there has been no work addressing unfairness in GNNs from the perspective of re-balancing.

To address the above problems, we propose a novel fair GNN model, \underline{Fair} \underline{G}raph Neural Networks via re-\underline{B}alancing, FairGB for short. Our proposed model can be divided into two modules: Counterfactual Node Mixup (CNM) and Contribution Alignment Loss (CAL). Specifically, we first select inter-domain (with the same target labels and different sensitive attributes) and inter-class (with different target labels and the same sensitive attributes) nodes for each training sample as counterexamples. Then we interpolate node attributes and neighbor distributions of counterfactual pairs and inject newly synthesized ego-networks to generate a balanced augmented graph. We have performed theoretical analysis over causal and statistical views, which serves as a solid mathematical foundation for the effectiveness of debiasing.  Secondly, because the importance of each training sample varies, achieving balance solely based on quantity is not sufficient. To further improve fairness, we propose a re-weighting method, Contribution Alignment Loss, which can balance the contribution of each group according to the gradients. The weights can be flexibly combined with CNM, thus we can view FairGB as a hybrid model, where CNM is equivalent to re-sampling, CAL is a re-weighting strategy. They mutually reinforce each other, helping to alleviate the issues encountered by the simple re-balancing methods mentioned above. Our contributions can be listed as follows:
\begin{itemize}[leftmargin=*]
    \item \textbf{Preliminary Analysis.} We find that simple re-balancing methods can easily achieve competitive or superior results compared to existing state-of-the-art fair GNN models. And we provide a new perspective to analyze the fairness in graph learning.
    \item \textbf{Algorithm Design.} We propose a novel approach, namely FairGB, that can effectively improve performance via re-balancing methods with only one additional hyper-parameter. And we theoretically prove that our approach can achieve the debiasing effect.
    \item \textbf{Experimental Evaluation.} We conduct extensive experiments, and the results demonstrate that FairGB achieves superior performance of utility and fairness. Meanwhile, we observe that the decision boundaries of target labels and sensitive attributes are roughly orthogonal, which indicates they are independent.
\end{itemize}
\section{Preliminary}

\subsection{Notations and Problem Statements}
Given an attributed graph $\mathcal{G}=(\mathcal{V},\mathbf{A},\mathbf{X})$, where $\mathcal{V}=\{v_1,v_2,\ldots,v_N\}$ is the set of nodes; $\mathbf{A}\in\mathbb{R}^{N\times N}$ is the adjacency matrix, $N$ is the number of nodes, if $v_i$ and $v_j$ are connected, $\mathbf{A}_{ij}=1$, otherwise $\mathbf{A}_{ij}=0$; $\mathbf{X}=[\mathbf{x}_1,\mathbf{x}_2,\ldots,\mathbf{x}_N]\in\mathbb{R}^{N\times D}$ is the node feature matrix, each node $v_i$ is associated with a $D$-dimensional node feature vector $\mathbf{x}_i$. The low-dimensional representations of nodes $\mathbf{Z}=[\mathbf{z}_1,\mathbf{z}_2.\ldots, \mathbf{z}_N]\in\mathbb{R}^{N\times d}$ are derived from graph encoder $g(\cdot):\mathbb{R}^{N\times D}\times\mathbb{R}^{N\times N}\rightarrow\mathbb{R}^{N\times d}$. On the top of the graph encoder, there is a classifier head $h(\cdot):\mathbb{R}^{N\times d}\rightarrow\mathbb{R}^{N\times C}$ to obtain the probability of each class, where $C$ is the number of classes. Combining graph encoder $g(\cdot)$ and classifier head $h(\cdot)$, we can acquire graph model $f_\theta=g\circ h(\cdot)$, where $\theta$ is the learnable parameters.

In fairness learning or debiased learning, each labeled sample is a triad, $(v,y,s)\sim\mathcal{V}\times\mathcal{Y}\times\mathcal{S}$, with $y$ being the ground-truth label and $s$ being the sensitive attributes. If the nodes have the same target label $y$ and sensitive attribute $s$, we define them as a demographic group $\mathcal{D}_{y,s}$, and the quantity of the group is $|\mathcal{D}_{y,s}|$. The goal of fairness learning is that the model will not be affected by the sensitive features, resulting in the bias of predicted results $\hat{y}$. Assume there is a binary classification task, target label $\mathcal{Y}=\{1, 0\}$, sensitive attribute $\mathcal{S}=\{1, 0\}$. There are two corresponding metrics to evaluate fairness.

\subsubsection{Demographic Parity} If the predicted result $\hat{y}$ is independent of sensitive attributes, \emph{i.e.}, $\hat{y}\bot s$, then we can consider demographic parity is achieved. The formula for this criterion is as follows:
\begin{equation}
    P(\hat{y}=1|s=0)=P(\hat{y}=1|s=1)
\end{equation}
If a model satisfies demographic parity, the acceptance rate of different protected groups is the same.

\subsubsection{Equalized Odds} If the predicted results and sensitive attributes are independent conditional on the ground-truth label, \emph{i.e.}, $\hat{y}\bot s|y$, then we consider equalized odds is achieved. The formula for it is as follows:
\begin{equation}
    P(\hat{y}=1|s=1,y=1)=P(\hat{y}=1|s=0,y=1)
\end{equation}
If a model satisfies equalized odds, the TPR (True Positive Rate) and FPR (False Positive Rate) for the two protected groups are the same.

\begin{figure*}[t]
\centering
\includegraphics[width=0.93\textwidth]{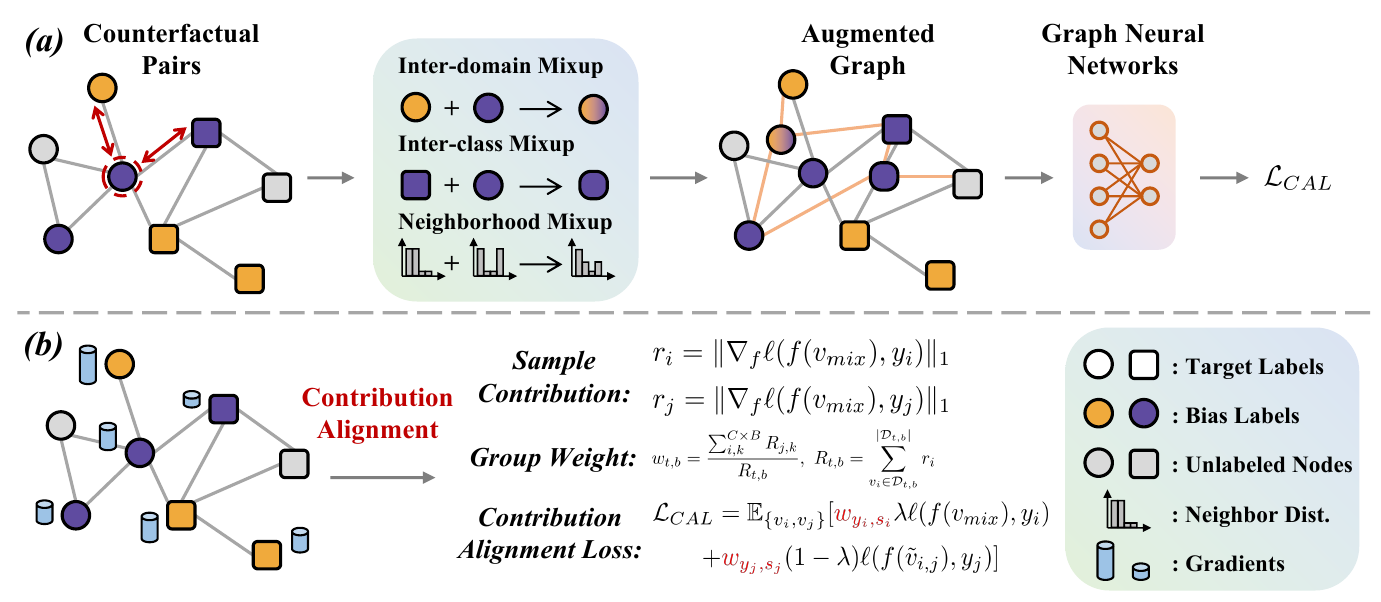}
\caption{The overview of proposed FairGB.}
\label{fig:architecture}
\end{figure*}

\subsection{Graph Neural Networks}
Most existing GNNs follow the message-passing paradigm which contains message aggregation and feature update, such as GCN \cite{kipf2016semi} and GAT \cite{velivckovic2017graph}. They generate node representations by iteratively aggregating information of neighbors and updating them with non-linear functions. The forward process can be defined as:
\begin{equation}
    \mathbf{z}_i^{(l)}=\mathbf{U}\Big(h_i^{(l-1)}, \mathbf{M}(\{\mathbf{z}_i^{(l-1)}, \mathbf{z}_j^{(l-1)}|v_j\in\mathcal{N}_i\})\Big)
\end{equation}
where $\mathbf{z}_i^{(l)}$ is the feature vector of node $i$ in the $l$-th layer, and $\mathcal{N}_i$ is a set of neighbor nodes of node $i$. $\mathbf{M}$ denotes the message passing function of aggregating neighbor information, $\mathbf{U}$ denotes the update function with central node feature and neighbor node features as input. By stacking multiple layers, GNNs can aggregate messages from higher-order neighbors.
\section{METHODOLOGY}

In this section, we will give a detailed description of FairGB. An illustration of its framework is shown in \figref{fig:architecture}. In the Counterfactual Node Mixup (CNM), we select counterfactual pairs for each training sample and conduct inter-domain and inter-class mixup with whole ego-networks. In the Contribution Alignment Loss (CAL), we further improve the group-wise balance by re-weighting each group according to the gradients. Next, we will provide a theoretical analysis and introduce the details of each module.

\subsection{Counterfactual Node Mixup}

Recently, the incorporation of causal learning techniques into GNNs has ignited a plethora of groundbreaking studies  \cite{jiang2023survey, guo2023towards, fan2022debiasing, sui2022causal}. This is attributed to the fact that addressing trustworthiness concerns is more effectively achieved by capturing the inherent causality in the underlying data, as opposed to merely relying on superficial correlations. In this work, we present a causal view of the union of the graph data generation and the GNNs' prediction process as a Structure Causal Model \cite{peters2017elements} (as shown in \figref{fig:analysis}(a)). We illustrate the causal relationships among six variables in the node classification problem: unobservable causal variable $C$, unobservable sensitive (bias) variable $S$, observable node attributes $\mathbf{X}$, observable topology $\mathbf{A}$, node embedding $\mathbf{Z}$, and ground-truth label $\mathbf{Y}$. In the graph data generation process, $C\rightarrow\mathbf{X}\leftarrow S$ and $C\rightarrow\mathbf{A}\leftarrow S$ demonstrate that two variables (causal variable $C$ and sensitive variable $S$) construct two components of observable contextual subgraphs (node attributes $\mathbf{X}$ and topology $\mathbf{A}$), which is different from \emph{i.i.d.} data only consider attributes. $\mathbf{X}\rightarrow\mathbf{Z}\leftarrow\mathbf{A}$ and $\mathbf{Z}\rightarrow\mathbf{Y}$ indicate existing GNNs produce representations and predictions based on observable contextual subgraphs. $C\dashleftrightarrow S$ denotes the spurious correlation between $C$ and $S$.

Inspired by \citet{fan2022debiasing}, we analyze the SCM according to $d$-connection theory \cite{pearl2009causality} (\emph{two variables are dependent if they are connected by at least one unblocked path}). Thus we can find three paths between sensitive variable $S$ and group-truth label $\mathbf{Y}$:
\begin{itemize}[leftmargin=*]
    \item $S\rightarrow\mathbf{X}\rightarrow\mathbf{Z}\rightarrow\mathbf{Y}$ and $S\rightarrow\mathbf{A}\rightarrow\mathbf{Z}\rightarrow\mathbf{Y}$: Because the existing graph neural networks make prediction for a node based on its contextual subgraph, sensitive variable $S$ will influence the final prediction not only through node attributes $\mathbf{X}$ but also topology $\mathbf{A}$ \cite{li2024gslb}. This makes fairness in graph machine learning more complex compared to other modalities (\emph{e.g.}, images and languages). As a result, if we want to sever all unblocked paths between $S$ and $\mathbf{Y}$, we need to debias both $\mathbf{X}$ and $\mathbf{A}$.
    
    \item $S\dashleftrightarrow C\rightarrow\mathbf{Y}$: we want to sever the connection between $S$ and $C$, so we utilize inter-domain and inter-class mixup. Intuitively, we interpolate samples with the same target label but different sensitive attributes in inter-domain mixup, and then the model can focus on class-specific information and learn domain invariant features. Besides, we also interpolate samples with the same sensitive attribute and different target labels in inter-class mixup, which can smooth the decision surface and alleviate the dependency on bias information.
\end{itemize}

Furthermore, we can also explain counterfactual mixup is equivalent to re-sampling methods. As shown in \figref{fig:analysis}(b), we observe that the original group distribution is imbalanced, which could lead to unfairness during training. However, in the inter-domain mixup, counterfactual samples can balance the bias distribution within each class (\emph{i.e.}, $P(Y=y|S=i)=P(Y=y|S=j), \forall i,j\in S$). In the inter-class mixup, counterfactual samples can make the bias distribution consistent within each class (\emph{i.e.}, $P(S=s|Y=i)=P(S=s|Y=j), \forall i,j\in Y$). The following theorem builds a theoretical analysis of the debiasing capability of counterfactual mixup.

\begin{figure}
    \centering
    \begin{minipage}{0.32\linewidth}
        \centering
        \includegraphics[width=0.96\linewidth]{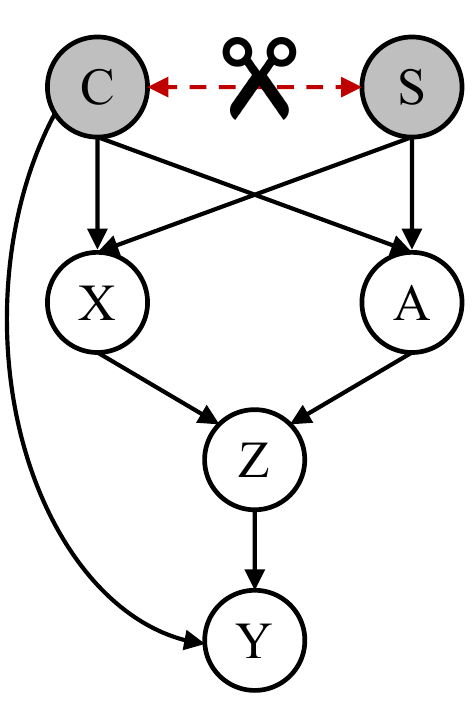}
        \caption*{(a) Causal View.}
    \end{minipage}
    \begin{minipage}{0.66\linewidth}
        \centering
        \includegraphics[width=0.93\linewidth]{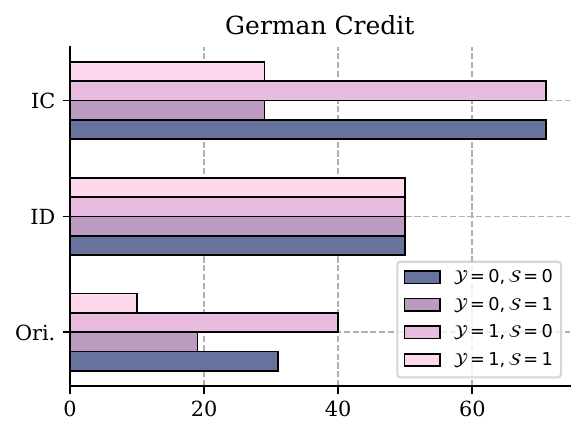}
        \caption*{(b) Group Distribution.}
    \end{minipage}
    \caption{Analysis of the proposed FairGB. (a) SCM of the union of the graph data generation and the existing GNNs' prediction process. (b) Statistics of the occurrences of samples in each group per epoch in the German dataset. "IC" denotes inter-class mixup, "ID" denotes inter-domain mixup.}
    \label{fig:analysis}
\end{figure}

\begin{theorem}
    Let $Y$ and $S$ be target labels and sensitive attributes. Balanced and consistent bias distribution within each class can make $S$ statistically independent from $Y$, i.e., $P(Y=y|S=s)=P(Y=y)$.
    \label{the:1}
\end{theorem}
\begin{proof}
    \ding{182}\textbf{Inter-domain mixup}: If $P(Y=y|S=i)=P(Y=y|S=j), \forall i,j\in S$, given $s\in S$ then:
    \begin{equation}
        P(Y=y)=\sum_{s\in S}P(Y=y|S=s)P(S=s),
        \label{equ:1}
    \end{equation}
    Based on the premise condition, we can rewrite \equref{equ:1} as:
    \begin{equation}
        P(Y=y)=P(Y=y|S=s)\sum_{s'\in S}P(S=s')=P(Y=y|S=s)
    \end{equation}
    \ding{183}\textbf{Inter-class mixup}: If $P(S=s|Y=i)=P(S=s|Y=j), \forall i,j\in Y$, given $y\in Y$ then:
    \begin{equation}
        P(S=s)=\sum_{y\in Y}P(S=s|Y=y)P(Y=y)
        \label{equ:2}
    \end{equation}
    Based on the premise condition, we can rewrite \equref{equ:2} as:
    \begin{equation}
        P(S=s)=P(S=s|Y=y)\sum_{y'\in Y}P(Y=y')=P(S=s|Y=y)
    \end{equation}
    Subsequently, using Bayesian probability we can get:
    \begin{align}
        P(Y=y|S=s)&=\frac{P(S=s|Y=y)P(Y=y)}{P(S=s)}\\
        &=\frac{P(S=s)P(Y=y)}{P(S=s)}\\
        &=P(Y=y)
    \end{align}
    Thus, we have completed the proof of \theref{the:1}.
\end{proof}

Although some literature provided similar theoretical analysis based on the statistical view \cite{qraitem2023bias, teney2023selective}, they did not conduct a detailed analysis of the debiasing mechanism of balanced and consistent bias distribution within each class.

Next, we are going to introduce the detailed process of counterfactual node mixup. Specifically, for each training node $v_i$, we randomly select a node $v_j$ that has the same target label and a different sensitive attribute with $v_i$ ($y_i=y_j, s_i\neq s_j$) as an inter-domain counterexample. In the same way, we randomly select a node $v_j$ that has a different target label and the same sensitive attribute with $v_i$ ($y_i\neq y_j, s_i=s_j$) as an inter-class counterexample. Here, we introduce a hyper-parameter $\eta$, which is responsible for controlling the ratio between two kinds of counterexamples. Then we perform linear interpolation on the node attributes:
\begin{equation}
    \mathbf{x}_{mix}=\lambda \mathbf{x}_i + (1-\lambda)\mathbf{x}_j,\ y_{mix}=\lambda y_i + (1-\lambda)y_j,
    \label{equ:node}
\end{equation}
where $\lambda\in [0,1]$ is the interpolation ratio that is sampled from a Beta distribution. After generating new node attributes, we need to insert these nodes into the original graph. According to the analysis above, we also perform mixup on the contextual subgraph structure. We define the neighbor distribution of node $v_i$ as $p_\mathcal{N}(v_i)$ where $p(v_k|v_i)=\frac{1}{|\mathcal{N}_i|}$, if $\mathbf{A}_{ik}=1$, and $p(v_k|v_i)=0$ otherwise. Then we directly interpolate the neighbor distribution of counterfactual pairs using the same $\lambda$:
\begin{equation}
    p_\mathcal{N}(v_{mix})=\lambda p_\mathcal{N}(v_i) + (1-\lambda)p_\mathcal{N}(v_j),
    \label{equ:neighbor}
\end{equation}
Thus we can obtain newly generated unbiased ego-networks. However, these ego-networks have very dense structures, which could violate the original degree distribution and result in the phenomenon of out-of-distribution \cite{tang2020investigating, liu2021tail}. Therefore, we sample the neighbors according to the original degree distribution and inject newly generated unbiased ego-networks into the original graph to construct an augmented graph $\mathcal{G}_{aug}$.

\subsection{Contribution Alignment Loss}

Although counterfactual node mixup is equivalent to re-sampling to some extent, varying levels of learning difficulty and the different positions of labeled nodes \cite{chen2021topology} lead to inconsistent contributions of training nodes. Simply balancing the group distribution in terms of quantity does not effectively address the problem. To further balance the contribution of each group and enhance fairness, we align the gradients of each group.

Since FairGB first generates mixed nodes through counterfactual node mixup and calculates the loss on these mixed nodes, it is difficult for us to determine which group the mixed nodes actually belong to. However, we can acquire the contributions of mixed nodes to each group by rewriting the loss function. First, we claim that the objective loss function of counterfactual mixup is as follows:
\begin{align}
    \mathcal{L}&=\mathbb{E}_{\{v_i,v_j\}}[\ell(f(v_{mix}), y_{mix})],\\
    &=\mathbb{E}_{\{v_i,v_j\}}[\lambda\ell(f(v_{mix}), y_i) + (1-\lambda)\ell(f(v_{mix}), y_j)],\label{equ:loss}
\end{align}
where $\ell$ is Cross-Entropy loss. Then we can obtain two contributions of each mixed node according to the gradients:
\begin{align}
    r_i=\lVert\nabla_f\ell(f(v_{mix}), y_i)\rVert_1,\ r_j=\lVert\nabla_f\ell(f(v_{mix}), y_j)\rVert_1,
    \label{equ:gradient}
\end{align}
In essence, $r_i$ and $r_j$ are the contributions generated from mixed two nodes $v_i$ and $v_j$. We can easily identify which group are $v_i$ and $v_j$ belong to. Therefore, the contribution of each group $R_{t,b}$ is the sum of sample contributions (both original samples and counterexamples), $R_{t,b}=\sum_{v_i\in\mathcal{D}_{t,b}}^{|\mathcal{D}_{t,b}|}r_i$. Based on $R_{t,b}$, we can compute the weight of each group $w_{t,b}=\sum_{j,k}^{C\times B}R_{j,k}/R_{t,b}$ to balance the contributions, where $B$ is the number of sensitive attributes. After obtaining the weights of each group, we can flexibly inject them into the objective loss function \equref{equ:loss} to get the final Contribution Alignment Loss (CAL):
\begin{equation}
    \mathcal{L}_{CAL}=\mathbb{E}_{\{v_i,v_j\}}[w_{y_i,s_i}\lambda\ell(f(v_{mix}),y_i)+w_{y_j,s_j}(1-\lambda)\ell(f(v_{mix}), y_j)],
    \label{equ:ca}
\end{equation}
So far, we have achieved an ingenious combination of counterfactual node mixup and contribution alignment. They can complement and promote each other. Specifically, counterfactual node mixup can sever the spurious correlation from the causal view but it does not guarantee a good balance between different groups. While contribution alignment loss is able to balance the contribution of each group based on gradients, but it will suffer the over-fitting problem, which can be mitigated by the newly generated unbiased ego-networks. The whole training procedure of FairGB is presented in \aloref{fairgb}.

\begin{algorithm}[t]
	\caption{FairGB: Fair Graph Neural Networks via re-Balancing;}
        \label{fairgb}
	\LinesNumbered
        \small
	\KwIn{An attributed graph: $\mathcal{G}=(\mathcal{V},\mathbf{A},\mathbf{X})$; hyper-parameters: $T, T_{warm}, \eta$; graph neural network: $f_\theta$.}
	\KwOut{The Learned node representations $\mathbf{Z}$ and predictions $\mathbf{\hat{Y}}$.}
    \For{$t=1, \ldots, T$}{
        \If{$t\leq T_{warm}$}{
            $\mathcal{L}\leftarrow$Cross Entropy Loss\;
            Back-propagation to update parameters\;
        }
        \Else{
            \For{$i=1,\ldots,|\mathcal{V}_{train}|$}{
                Sample $\mu\sim\text{Uniform}(0,1)$\;
                \If{$\mu\geq\eta$}{
                    Randomly sample $v_j$ which satisfies $y_i=y_j$ and $s_i\neq s_j$;\ \ \tcp{Inter-domain}
                }
                \Else{
                    Randomly sample $v_j$ which satisfies $y_i\neq y_j$ and $s_i=s_j$;\ \ \tcp{Inter-class}
                }
                Generate mixed ego-network according to \equref{equ:node} and \equref{equ:neighbor}\;
                Compute gradients by \equref{equ:gradient}\;
            }
            Compute group weights $w_{t,s}$\;
            $\mathcal{L}_{CAL}\leftarrow$Contribution Alignment Loss (\equref{equ:ca})\;
            Back-propagation to update parameters\;
        }
    }
\end{algorithm}

\section{Experiments}
\label{sec:exp}

In this section, we conduct extensive experiments to investigate the effectiveness of our proposed model, and aim to answer the following research questions:
\begin{itemize}[leftmargin=*]
    \item \textbf{RQ1}: Compared to other baselines, can FairGB achieve better performance \emph{w.r.t.} utility and fairness?
    \item \textbf{RQ2}: How does each component affect the model performance?
    \item \textbf{RQ3}: What are the characteristics of the node features generated by FairGB compared to vanilla?
    \item \textbf{RQ4}: Can FairGB generalize well to different graph encoders?
    \item \textbf{RQ5}: How do hyper-parameters affect FairGB?
\end{itemize}

\subsection{Experimental settings}
\subsubsection{Real-World Datasets.} We conduct experiments on three widely used real-world datasets, namely German Credit, Bail, and Credit Defaulter. The statistics of the datasets can be found in \tabref{tab:dataset}. The details of the datasets are as follows:
\begin{itemize}[leftmargin=*]
    \item \textbf{German Credit} \cite{asuncion2007uci}: the nodes in the dataset are clients and two nodes are connected if they have a high similarity of the credit accounts. The task is to classify the credit risk level as high or low with the sensitive attribute "gender".
    \item \textbf{Bail} \cite{jordan2015effect}: these datasets contain defendants released on bail during 1990-2009 as nodes. The edges between the two nodes are connected based on the similarity of past criminal records and demographics. The task is to classify whether defendants are on bail or not with the sensitive attribute "race".
    \item \textbf{Credit Defaulter} \cite{yeh2009comparisons}: the nodes in the dataset are credit card users and the edges are formed based on the similarity of the payment information. The task is to classify the default payment method with the sensitive attribute "age".
\end{itemize}

\begin{table}[]
\caption{The Statistic of Datasets.}
\resizebox{1\linewidth}{!}{
\begin{tabular}{lcccccc}
\toprule
\textbf{Dataset}       & \multicolumn{2}{c}{\textbf{German Credit}} & \multicolumn{2}{c}{\textbf{Bail}} & \multicolumn{2}{c}{\textbf{Credit Defaulter}} \\ \midrule
\# Nodes      & \multicolumn{2}{c}{1,000}              & \multicolumn{2}{c}{18,876}     & \multicolumn{2}{c}{30,000}                 \\
\# Edges      & \multicolumn{2}{c}{22,242}              & \multicolumn{2}{c}{321,308}     & \multicolumn{2}{c}{152,377}                 \\
\# Attributes & \multicolumn{2}{c}{27}              & \multicolumn{2}{c}{18}     & \multicolumn{2}{c}{13}                 \\
Sens.         & \multicolumn{2}{c}{Gender}              & \multicolumn{2}{c}{Race}     & \multicolumn{2}{c}{Age}                 \\
Label         & \multicolumn{2}{c}{Credit status}              & \multicolumn{2}{c}{Bail decision}     & \multicolumn{2}{c}{Future default}                 \\ 
\bottomrule
\end{tabular}}
\label{tab:dataset}
\end{table}

\begin{table*}[]
\caption{Model performance on German, Credit, and Bail with respect to utility and fairness. Dark brown is used to highlight the best results for each metric. Light brown for the runner-up results. $\uparrow$ represents the larger, the better while $\downarrow$ represents the smaller, the better. Each experimental result is obtained from 10 repeated experiments.}
\resizebox{1\textwidth}{!}{
\begin{tabular}{lc|ccc|cc|cc|cc|c}
\toprule
\multicolumn{1}{l|}{\textbf{Datasets}}                 & \textbf{Metrics} & \textbf{GCN} & \textbf{GIN} & \textbf{SAGE} & \textbf{FairGNN} & \textbf{EDITS} & \textbf{NIFTY} & \textbf{CAF} & \textbf{RW} & \textbf{OS} & \textbf{FairGB} \\ \midrule
\multicolumn{1}{l|}{\multirow{5}{*}{\textbf{German}}} & AUC ($\uparrow$)        & 73.49{\small$\pm$2.15}    & 72.42{\small$\pm$1.46}    & \second73.78{\small$\pm$1.80}     & 65.85{\small$\pm$9.49}        & 69.76{\small$\pm$5.46}      & 72.05{\small$\pm$2.15}      & 71.87{\small$\pm$1.33}    & 71.86{\small$\pm$2.18}   & 71.73{\small$\pm$2.92}   & \best74.45{\small$\pm$4.93}     \\

\multicolumn{1}{l|}{}                        & F1 ($\uparrow$)        & 78.96{\small$\pm$3.35}    & 81.76{\small$\pm$1.27}    & 81.01{\small$\pm$1.18}     & \second82.29{\small$\pm$0.32}        & 81.04{\small$\pm$1.09}      & 79.20{\small$\pm$1.19}      & 82.16{\small$\pm$0.22}    & 81.52{\small$\pm$0.77}   & 82.01{\small$\pm$0.65}   & \best82.58{\small$\pm$0.35}     \\

\multicolumn{1}{l|}{}                        & ACC ($\uparrow$)        & 70.84{\small$\pm$2.56}    & \best72.36{\small$\pm$0.77}    & \second71.76{\small$\pm$1.15}     & 70.64{\small$\pm$0.74}        & 71.68{\small$\pm$1.25}      & 69.60{\small$\pm$1.50}      & 70.64{\small$\pm$0.34}    & 70.44{\small$\pm$0.87}   & 70.00{\small$\pm$1.72}   & 70.76{\small$\pm$0.79}     \\

\multicolumn{1}{l|}{}                        & $\Delta_{sp}$ ($\downarrow$)        & 34.95{\small$\pm$15.33}    & 16.65{\small$\pm$8.41}    & 28.17{\small$\pm$4.99}     & 7.65{\small$\pm$8.07}        & 8.42{\small$\pm$7.35}      & 7.74{\small$\pm$7.80}      & 6.60{\small$\pm$1.66}    & 4.80{\small$\pm$4.50}   & \second4.02{\small$\pm$3.87}   & \best2.19{\small$\pm$1.89}     \\

\multicolumn{1}{l|}{}                        & $\Delta_{eo}$ ($\downarrow$)        & 28.43{\small$\pm$11.66}    & 11.79{\small$\pm$7.29}    & 19.40{\small$\pm$6.11}     & 4.18{\small$\pm$4.86}        & 5.69{\small$\pm$2.16}      & 5.17{\small$\pm$2.38}      & \second1.58{\small$\pm$1.14}    & 2.24{\small$\pm$2.64}   & 2.66{\small$\pm$3.30}   & \best1.20{\small$\pm$1.37}     \\ \midrule

\multicolumn{1}{l|}{\multirow{5}{*}{\textbf{Bail}}}   & AUC ($\uparrow$)        & 87.09{\small$\pm$0.18}    & 85.68{\small$\pm$0.25}    & 92.07{\small$\pm$0.83}     & 91.53{\small$\pm$0.38}        & 89.07{\small$\pm$2.26}      & 92.04{\small$\pm$0.89}      & 91.39{\small$\pm$0.34}    & 91.54{\small$\pm$0.47}   & \second93.37{\small$\pm$0.54}   & \best96.21{\small$\pm$1.26}     \\

\multicolumn{1}{l|}{}                        & F1 ($\uparrow$)        & 77.93{\small$\pm$0.39}    & 76.92{\small$\pm$0.31}    & 83.64{\small$\pm$0.95}     & 82.55{\small$\pm$0.98}        & 77.83{\small$\pm$3.79}      & 77.81{\small$\pm$6.03}      & 83.09{\small$\pm$0.98}    & 82.73{\small$\pm$0.96}   & \second85.60{\small$\pm$1.20}   & \best90.19{\small$\pm$1.76}     \\

\multicolumn{1}{l|}{}                        & ACC ($\uparrow$)        & 83.58{\small$\pm$0.31}    & 82.34{\small$\pm$0.29}    & 88.61{\small$\pm$0.68}     & 87.68{\small$\pm$0.73}        & 84.42{\small$\pm$2.87}      & 84.11{\small$\pm$5.49}      & 88.02{\small$\pm$0.86}    & 87.20{\small$\pm$1.04}   & \second89.36{\small$\pm$1.13}   & \best92.64{\small$\pm$1.21}     \\

\multicolumn{1}{l|}{}                        & $\Delta_{sp}$ ($\downarrow$)        & 7.59{\small$\pm$0.31}    & 8.78{\small$\pm$0.40}    & 3.84{\small$\pm$1.15}     & 1.94{\small$\pm$0.82}        & 3.74{\small$\pm$3.54}      & 5.74{\small$\pm$0.38}      & 2.29{\small$\pm$1.06}    & \best0.68{\small$\pm$0.59}   & 1.70{\small$\pm$1.76}   & \second0.77{\small$\pm$0.52}     \\

\multicolumn{1}{l|}{}                        & $\Delta_{eo}$ ($\downarrow$)        & 5.26{\small$\pm$0.30}    & 7.47{\small$\pm$0.39}    & 2.64{\small$\pm$1.31}     & 1.72{\small$\pm$0.70}        & 4.46{\small$\pm$3.50}      & 4.07{\small$\pm$1.28}      & \second1.17{\small$\pm$0.52}    & \best0.98{\small$\pm$0.74}   & 1.53{\small$\pm$1.01}   & 1.53{\small$\pm$0.65}     \\ \midrule

\multicolumn{1}{l|}{\multirow{5}{*}{\textbf{Credit}}} & AUC ($\uparrow$)        & 73.90{\small$\pm$0.03}    & 73.20{\small$\pm$0.02}    & 74.55{\small$\pm$0.60}     & 70.82{\small$\pm$0.74}        & \best75.04{\small$\pm$0.12}      & 72.89{\small$\pm$0.44}      & 73.42{\small$\pm$1.89}    & 73.38{\small$\pm$0.56}   & \second74.76{\small$\pm$0.17}   & 73.07{\small$\pm$1.79}     \\

\multicolumn{1}{l|}{}                        & F1 ($\uparrow$)        & 81.94{\small$\pm$0.01}    & 83.18{\small$\pm$0.10}    & 83.95{\small$\pm$1.19}     & 83.97{\small$\pm$2.00}        & 82.41{\small$\pm$0.52}      & 82.60{\small$\pm$1.25}      & 83.63{\small$\pm$0.89}    & 83.43{\small$\pm$1.19}   & \second86.91{\small$\pm$0.70}   & \best87.47{\small$\pm$0.52}     \\

\multicolumn{1}{l|}{}                        & ACC ($\uparrow$)        & 73.69{\small$\pm$0.01}    & 75.01{\small$\pm$0.12}    & 75.82{\small$\pm$1.04}     & 75.29{\small$\pm$1.62}        & 74.13{\small$\pm$0.59}      & 74.39{\small$\pm$1.35}      & 75.36{\small$\pm$0.95}    & 75.22{\small$\pm$1.36}   & \second79.14{\small$\pm$0.77}   & \best79.38{\small$\pm$0.52}     \\

\multicolumn{1}{l|}{}                        & $\Delta_{sp}$ ($\downarrow$)        & 12.73{\small$\pm$0.15}    & 5.41{\small$\pm$0.43}    & 15.91{\small$\pm$2.83}     & 6.17{\small$\pm$5.57}        & 11.34{\small$\pm$6.36}      & 10.65{\small$\pm$1.65}      & 8.63{\small$\pm$2.13}    & 4.63{\small$\pm$1.41}   & \second3.85{\small$\pm$2.15}   & \best1.61{\small$\pm$1.41}     \\

\multicolumn{1}{l|}{}                        & $\Delta_{eo}$ ($\downarrow$)        & 10.53{\small$\pm$0.13}    & 3.45{\small$\pm$0.41}    & 13.45{\small$\pm$3.17}     & 5.06{\small$\pm$4.46}        & 9.38{\small$\pm$5.39}      & 8.10{\small$\pm$1.91}      & 6.85{\small$\pm$1.55}    & 2.84{\small$\pm$0.99}   & \second2.23{\small$\pm$1.13}   & \best0.92{\small$\pm$0.83}     \\ \midrule

\multicolumn{2}{c|}{\textbf{Avg. (Rank)}}                              & 8.13    & 7.00    & 5.47     & 5.47        & 7.00      & 7.40      & 4.67    & 4.60   & \second3.27   & \best2.00     \\ \bottomrule
\end{tabular}}
\label{tab:compare}
\end{table*}

\subsubsection{Baselines.} We compare our proposed model with 9 representative and state-of-the-art methods in four categories, which include: (1) \textbf{Vanilla graph neural networks}: GCN \cite{kipf2016semi} is widely used spectral GNN; GraphSAGE (SAGE for short) \cite{hamilton2017inductive} is a method for inductive learning that leverages node feature information to generate embedding for nodes in large graph; GIN \cite{xu2018powerful} is a graph-based neural network that can capture different topological structures by injecting the node's identity into its aggregation function. (2) \textbf{Fair node classification methods}: FairGNN \cite{dai2021say} uses adversarial training to achieve fairness on graphs; EDITS \cite{dong2022edits} is a pre-processing method for fair graph learning. (3) \textbf{Graph counterfactual fairness methods}: NIFTY \cite{agarwal2021towards} simply performs a flipping on the sensitive attributes to get counterfactual data; CAF \cite{guo2023towards} is guided by causal analysis, which can select counterfactual from training data to avoid non-realistic counterfactuals and adopt selected counterfactuals to learn fair node representations. (4) \textbf{Simple re-balancing methods}: Re-weighting (RW for short) up-weights the contribution of minority groups and down-weights the contribution of majority groups to the loss functions according to the quantity of groups. More concretely, the objective loss function is $\mathcal{L}=\frac{1}{N}\sum_{i=1}^N\frac{N}{|\mathcal{D}_{y_i,s_i}|}\ell(f(v_i), y_i)$.
Over-sampling (OS for short) repeatedly samples minority group samples until the number of each group data reaches the maximum number of group data. We duplicate the edges of the original node when adding an oversampled node to the original graph. Following the setting of CAF \cite{guo2023towards}, we also use SAGE as the model backbone except for GCN and GIN.

\subsubsection{Evaluation Metrics.} We regard AUC, F1 score, and accuracy as utility metrics. For fairness metrics, we use statistical parity (SP) $\Delta_{sp}$ and equal opportunity (EO) $\Delta_{eo}$, a smaller fairness metric indicates a fairer model decision.

\subsubsection{Implementation details.} For German, Bail, and Credit datasets, we follow train/valid/test split in \cite{agarwal2021towards}. The hyper-parameters used in experiments follow the source codes or are searched by the grid search method, and we use the Adam optimization algorithm \cite{kingma2014adam} to train all the models. Specifically, FairGB only has one additional hyper-parameter $\eta$, and it is searched from \{0, 0.1, 0.2, $\ldots$, 0.8, 0.9, 1\}. All the models are implemented in PyTorch \cite{paszke2019pytorch} version 2.0.1 with PyTorch Geometric \cite{fey2019fast} version 2.3.1.

\subsection{RQ1: Performance comparison}
To comprehensively understand the effectiveness of FairGB, we conduct node classification on three widely used datasets. The experimental results of utility and fairness of each model are shown in \tabref{tab:compare}. From the \tabref{tab:compare}, our observations can be threefold: (1) We can observe that simple re-balancing methods achieve satisfactory results in both utility and fairness on three datasets. They can obtain competitive or even superior performance compared to carefully designed fair GNN models. (2) FairGB consistently achieves the best performance on the utility-fairness trade-off on all datasets. We use the average rank of three utility metrics and two fairness metrics to understand the performance of the trade-off. Our model ranks 2.47 and the runner-up model ranks 3.27. Our model outperforms all four categories of baselines, which shows the effectiveness of our model. (3) We find that OS and FairGB, on bail and credit datasets, do not sacrifice utility while improving fairness. For instance, on the bail dataset, FairGB improves the AUC by 4.50\%, the F1 by 7.83\%, and Acc by 4.54\% compared to the vanilla model. And on the credit dataset, FairGB improves F1 by 4.19\%, and Acc by 4.70\% compared to the vanilla model. We speculate that this may be due to two reasons: (i) Balance of contributions from each group leads to a significant improvement in the accuracy of the worst group; (ii) the re-balancing strategy enhances the model's generalization, resulting in an overall improvement in classification performance.

\begin{figure*}[t]
\centering
\includegraphics[width=\textwidth]{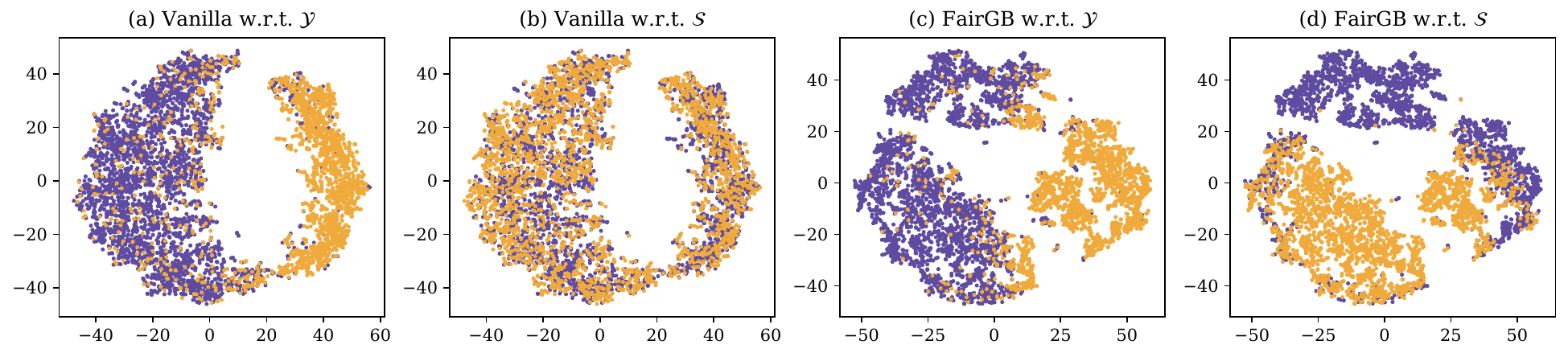}
\caption{Visualizations of node representation learned on the Bail dataset.}
\label{fig:representation}
\end{figure*}

\subsection{RQ2: Ablation study}
To answer \textbf{RQ2} and verify the effectiveness of our proposed FairGB, we construct two variants of FairGB: (1) without Contribution Alignment Loss but conduct node mixup within counterfactual pairs (called FairGB w/o CAL); (2) without Counterfactual Node Mixup but assigns weights based on the group contribution (called FairGB w/o CNM). \tabref{tab:ablation} demonstrates the performance of the vanilla model, two variants, and FairGB. First, we observe that two variants perform worse than FairGB in the trade-off of utility and fairness, which proves the effectiveness of each component and the rationality of the combination. Second, we find that FairGB w/o CNM consistently achieves better fairness compared to FairGB w/o CAL, which indicates the importance of re-balancing in fair graph learning. However, as mentioned above, re-weighting methods will encounter over-fitting problems, and counterfactual node mixup will generate new samples per epoch, which could mitigate over-fitting. Third, while both modules effectively enhance fairness, the impact on utility varies across different datasets for each variant. For example, FairGB w/o CAL achieves high utility on the bail dataset but does not show improvement on the credit dataset, whereas FairGB w/o CNM exhibited the opposite trend. We speculate that this may be due to the different sizes of labeled nodes in the two datasets. The bail dataset has a small number of training samples (only 100 nodes are labeled), leading to over-fitting issues with re-weighting. Mixup, on the other hand, is beneficial for training with few labels. In the credit dataset, there are a lot of training samples (4000 nodes are labeled), and the phenomenon is just the opposite. However, we can observe that FairGB competently incorporates the strengths of each module, demonstrating that the two modules can mutually reinforce each other.

\begin{table}[]
\caption{Ablation study results.}
\resizebox{1\linewidth}{!}{
\begin{tabular}{l|c|cccc}
\toprule
\multirow{2}{*}{\textbf{Datasets}} & \multirow{2}{*}{\textbf{Metrics}} & \multirow{2}{*}{\textbf{Vanilla}} & \textbf{FairGB} & \textbf{FairGB}  & \multirow{2}{*}{\textbf{FairGB}} \\
                          &                          &                          & \textbf{w/o CAL} & \textbf{w/o CNM} &                         \\ \midrule
\multirow{5}{*}{\textbf{German}}   & AUC ($\uparrow$)                     & 73.72\textcolor{color2}{\scriptsize{-0.72\%}}                         & \textbf{75.09}\textcolor{color1}{\scriptsize{+0.64\%}}       & 69.29\textcolor{color2}{\scriptsize{-5.16\%}}        & 74.45                        \\
                          & F1 ($\uparrow$)                      & 81.01\textcolor{color2}{\scriptsize{-1.57\%}}                         & 82.46\textcolor{color2}{\scriptsize{-0.12\%}}       & 81.38\textcolor{color2}{\scriptsize{-1.20\%}}        & \textbf{82.58}                        \\
                          & ACC ($\uparrow$)                     & 71.76\textcolor{color1}{\scriptsize{+1.00\%}}                         & \textbf{71.80}\textcolor{color1}{\scriptsize{+1.04\%}}       & 70.36\textcolor{color2}{\scriptsize{-0.40\%}}        & 70.76                       \\
                          & $\Delta_{sp}$ ($\downarrow$)                       & 28.17\textcolor{color2}{\scriptsize{+25.98\%}}                         & 7.14\textcolor{color2}{\scriptsize{+4.95\%}}       & 4.10\textcolor{color2}{\scriptsize{+1.91\%}}        & \textbf{2.19}                        \\
                          & $\Delta_{eo}$ ($\downarrow$)                       & 19.40\textcolor{color2}{\scriptsize{+18.20\%}}                         & 2.77\textcolor{color2}{\scriptsize{+1.57\%}}       & 2.77\textcolor{color2}{\scriptsize{+1.57\%}}        & \textbf{1.20}                        \\ \midrule
\multirow{5}{*}{\textbf{Bail}}     & AUC ($\uparrow$)                     & 92.07\textcolor{color2}{\scriptsize{-4.14\%}}                         & \textbf{97.11}\textcolor{color1}{\scriptsize{+0.90\%}}       & 91.13\textcolor{color2}{\scriptsize{-5.08\%}}        & 96.21                        \\
                          & F1 ($\uparrow$)                      & 83.64\textcolor{color2}{\scriptsize{-6.55\%}}                         & \textbf{91.93}\textcolor{color1}{\scriptsize{+1.74\%}}       & 82.36\textcolor{color2}{\scriptsize{-7.83\%}}        & 90.19                        \\
                          & ACC ($\uparrow$)                     & 88.61\textcolor{color2}{\scriptsize{-4.03\%}}                         & \textbf{93.99}\textcolor{color1}{\scriptsize{+1.35\%}}       & 87.19\textcolor{color2}{\scriptsize{-5.45\%}}        & 92.64                        \\
                          & $\Delta_{sp}$ ($\downarrow$)                       & 3.84\textcolor{color2}{\scriptsize{+3.07\%}}                         & 1.36\textcolor{color2}{\scriptsize{+0.59\%}}       & 1.33\textcolor{color2}{\scriptsize{+0.56\%}}        & \textbf{0.77}                        \\
                          & $\Delta_{eo}$ ($\downarrow$)                       & 2.64\textcolor{color2}{\scriptsize{+1.11\%}}                         & 1.82\textcolor{color2}{\scriptsize{+0.29\%}}       & \textbf{1.06}\textcolor{color1}{\scriptsize{-0.47\%}}        & 1.53                       \\ \midrule
\multirow{5}{*}{\textbf{Credit}}   & AUC ($\uparrow$)                     & \textbf{74.55}\textcolor{color1}{\scriptsize{+1.48\%}}                         & 73.61\textcolor{color1}{\scriptsize{+0.54\%}}       & 73.77\textcolor{color1}{\scriptsize{+0.70\%}}        & 73.07                        \\
                          & F1 ($\uparrow$)                      & 83.95\textcolor{color2}{\scriptsize{-3.52\%}}                         & 82.86\textcolor{color2}{\scriptsize{-4.61\%}}       & 86.18\textcolor{color2}{\scriptsize{-1.29\%}}        & \textbf{87.47}                        \\
                          & ACC ($\uparrow$)                     & 75.82\textcolor{color2}{\scriptsize{-3.56\%}}                         & 74.48\textcolor{color2}{\scriptsize{-4.90\%}}       & 78.19\textcolor{color2}{\scriptsize{-1.19\%}}        & \textbf{79.38}                        \\
                          & $\Delta_{sp}$ ($\downarrow$)                       & 15.91\textcolor{color2}{\scriptsize{+12.59\%}}                         & 6.84\textcolor{color2}{\scriptsize{+3.52\%}}       & 5.76\textcolor{color2}{\scriptsize{+2.44\%}}        & \textbf{3.32}                        \\
                          & $\Delta_{eo}$ ($\downarrow$)                       & 13.45\textcolor{color2}{\scriptsize{+11.96\%}}                         & 4.15\textcolor{color2}{\scriptsize{+2.66\%}}       & 3.86\textcolor{color2}{\scriptsize{+2.37\%}}        & \textbf{1.49}                        \\ \bottomrule
\end{tabular}}
\label{tab:ablation}
\vspace{-0.2cm}
\end{table}

\subsection{RQ3: Visualization}
In order to answer the \textbf{RQ3}, we visualize the learned node embeddings on the bail dataset to better understand the mechanism of FairGB. We compare our FairGB with the vanilla GNN model (\emph{i.e.} SAGE), and then use 16-dimensional output embedding of the encoder. Subsequently, we use t-SNE \cite{van2008visualizing} to map the 16-dimensional embedding into 2-dimensional space for visualization. We only select samples in the test set for better visibility. The results are shown in \figref{fig:representation}. We plot two figures for each model, one concerns the target labels, and the other one concerns the sensitive attributes. Comparing \figref{fig:representation}(a) and \figref{fig:representation}(c), we can observe that the representations of nodes from different classes have smaller overlapping regions in FairGB. This confirms the improvement of FairGB on the three utility metrics. Next, since SAGE already exhibits good fairness on the bail dataset (as shown in \tabref{tab:compare}), we can observe that the node representations of the two sensitive attributes are mixed together in \figref{fig:representation}(b). Because FairGB is designed from a re-balancing perspective, the contributions from each group are relatively balanced, making the node representations of the two sensitive attributes distinguishable in FairGB. However, we can observe that the classification boundary for the target label is orthogonal to the classification boundary for the sensitive features. This explains how FairGB achieves excellent fairness.

\subsection{RQ4: Generalization for different encoders}

To answer the \textbf{RQ4}, we test the generalization ability of our FairGB by deploying it on three different graph encoders: SAGE, GCN, and GIN. The utility and fairness performance are demonstrated in \figref{fig:encoder}. We can observe that FairGB is able to maintain or even surpass the vanilla models in both AUC and F1-score metrics, which is thanks to the re-balancing strategy. In the fairness metrics, we find that FairGB effectively reduces $\Delta_{sp}$ and $\Delta_{eo}$ compared to the vanilla models across all datasets. These results indicate that FairGB has strong generalization capabilities for different graph encoders, making it flexible for use in real-world applications.

\subsection{RQ5: Parameter sensitive analysis}

One major advantage of FairGB is that it only has one additional hyper-parameter $\eta$ to control the trade-off between inter-domain mixup and inter-class mixup. To figure out the \textbf{RQ5}, we conduct hyper-parameter sensitive analysis on three datasets in terms of three utility metrics and two fairness metrics. As shown in \figref{fig:hyper}, we vary $\eta$ from 0 to 1 and present the performance of FairGB. When $\eta$ equals 0, the model only performs inter-domain mixup, and when $\eta$ equals 1, the model only performs inter-class mixup. Recall our above analysis, when the training set is class-balanced, inter-domain mixup is equivalent to group-wise balance re-sampling. But according to our causal view, we demonstrate that plain inter-domain mixup can not effectively mitigate bias, so we conduct inter-class mixup and add a hyper-parameter $\eta$ to control the trade-off. We can observe that when $\eta$ equals 0 or 1, FairGB can not achieve the best performance. However, when $\eta$ is within the range of 0.3 to 0.8, the model is not sensitive to parameters and FairGB achieves the roughly best performance when $\eta$ equals 0.5. Therefore, we simply fix the $\eta$ to 0.5 in most experiments.
\section{Related Work}
\label{sec:related}

\begin{figure}[t]
\centering
\includegraphics[width=\linewidth]{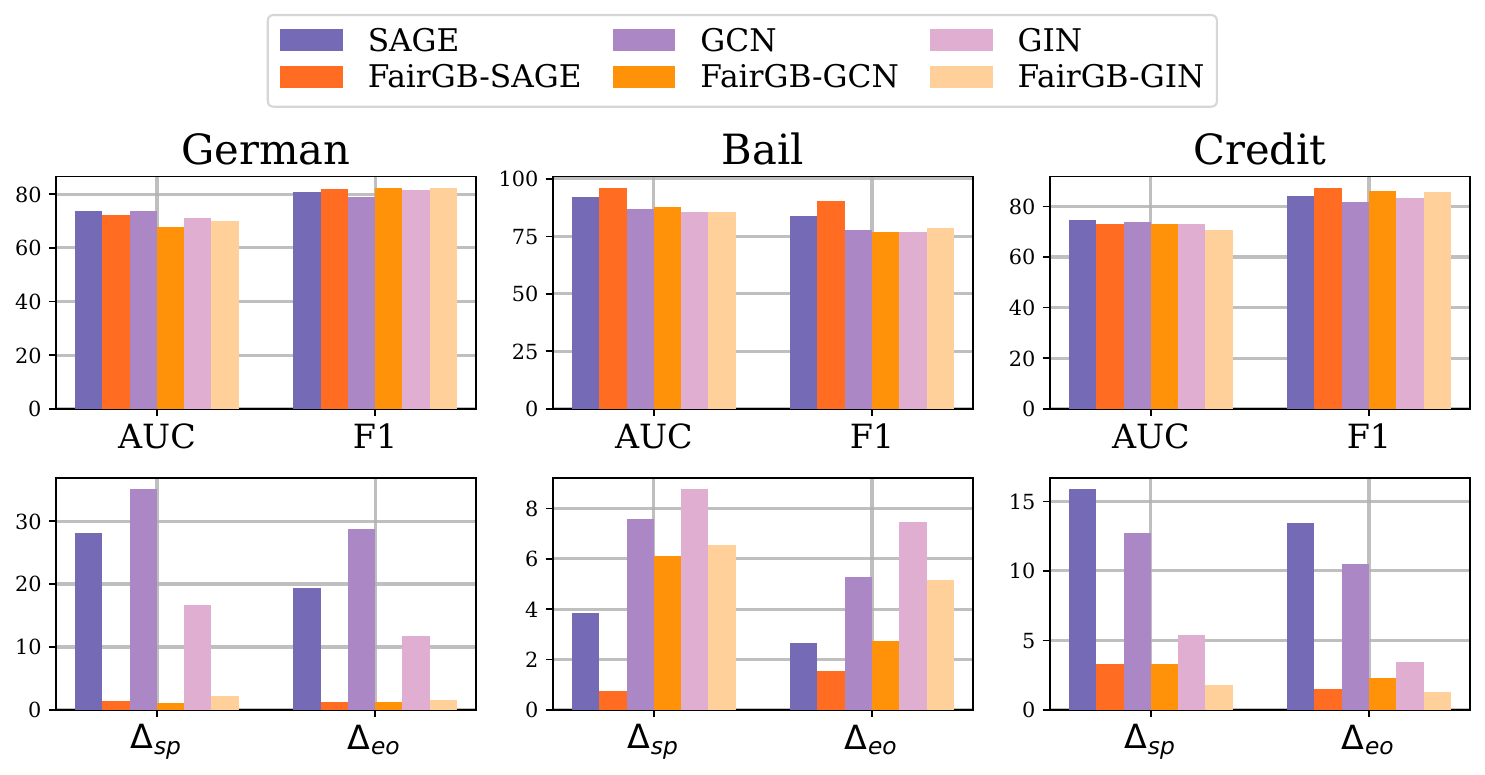}
\vspace{-3mm}
\caption{Comparison of the utility and fairness performance with different graph encoders.}
\label{fig:encoder}
\end{figure}

\begin{figure}[t]
\centering
\vspace{-3mm}
\includegraphics[width=\linewidth]{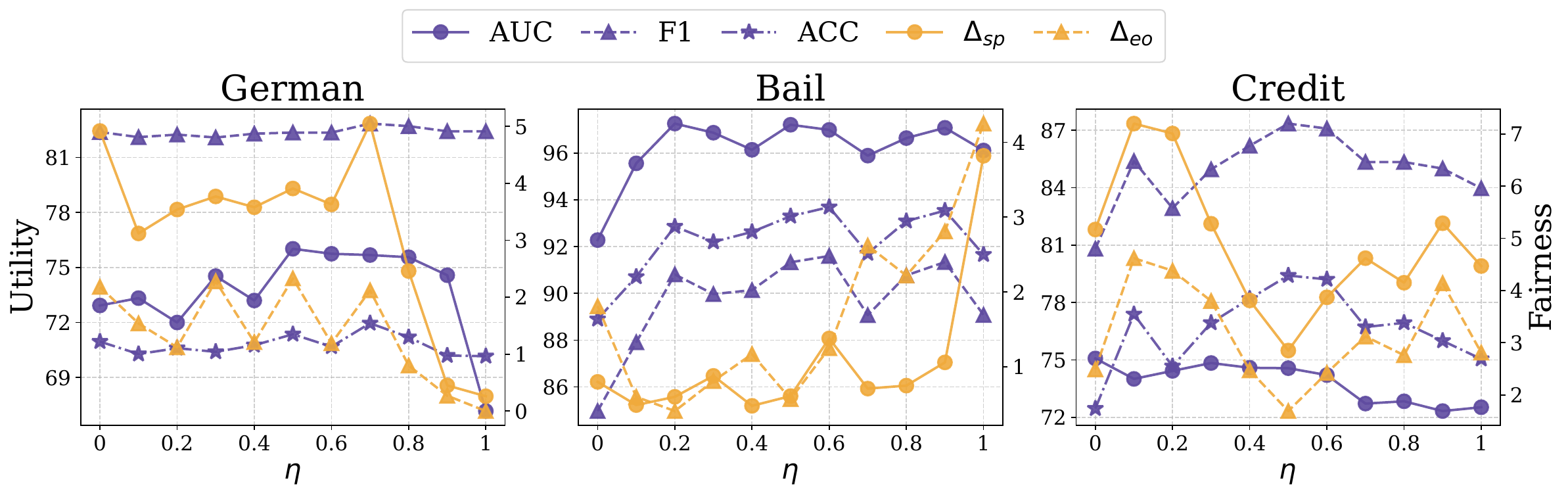}
\vspace{-3mm}
\caption{Parameter sensitivity results \emph{w.r.t.} $\eta$.}
\vspace{-3mm}
\label{fig:hyper}
\end{figure}

\subsection{Fairness in Graph Neural Networks}

Fairness is a widely researched issue in the field of machine learning \cite{mehrabi2021survey, berk2019accuracy, beutel2017data, chen2018my}. Most works in deep learning exclusively focus on optimizing the model utility while ignoring the fairness of the decisions, such as group fairness \cite{diana2021minimax, beutel2017data}, individual fairness \cite{mukherjee2020two, sharifi2019average}, and counterfactual fairness \cite{kusner2017counterfactual, chiappa2019path}. Since GNNs inherit characteristics from machine learning, they can also encounter fairness issues, which makes it challenging to deploy GNNs in high-risk applications. 

Nowadays, there are various existing works that try to improve the fairness of GNNs, and they can be roughly categorized into pre-, in-, and post-processing methods. Pre-processing methods modify the original training data with the assumption that fair data would result in fair models. Fairwalk \cite{rahman2019fairwalk} and Crosswalk \cite{khajehnejad2022crosswalk} choose each group of neighbor nodes with an equal chance and bias random walks to cross group boundaries. EDITS \cite{dong2022edits} designs a model-agnostic method to modify the attribute and structure for fair GNNs training. In-processing approaches aim to mitigate unfairness during the training process by directly modifying the learning algorithm, and they can be divided into three parts: adversarial-based, augmentation-based, and message-passing-based. Adversarial-based approaches train fair GNNs by preventing an adversary from correctly predicting sensitive attributes from the learned node representations. FairGNN \cite{dai2021say} and FairVGNN \cite{wang2022improving} enforce the model to generate fair outputs with adversarial training through the min-max objective. Augmentation-based approaches generate counterfactual views and minimize the discrepancy with the original view. NIFTY \cite{agarwal2021towards} flips sensitive attributes for each node to obtain the counterfactual view. Message-passing-based approaches improve fairness through the view of optimization problem with smoothness regularization. FMP \cite{jiang2022fmp} proposes a fair message-passing framework by considering graph smoothness and fairness objectives. The post-processing techniques directly calibrate the classifier's decisions at inference time. P{\small OST}P{\small ROCESS} \cite{merchant2023disparity} updates model predictions based on a black-box policy to minimize differences between demographic groups.

\subsection{Re-balancing in Graph Neural Networks}

Due to the GNNs inheriting the character of deep neural networks, GNNs perform with biases toward the majority classes when training on imbalanced datasets. To overcome this challenge, class-imbalanced learning on graphs has emerged as a promising solution that combines the strengths of graph representation learning and class-imbalanced learning. A great branch of these methods is over-sampling minority nodes by data augmentation to balance the skew label distribution. GraphSMOTE \cite{zhao2021graphsmote} leverages representative data augmentation method (\emph{i.e.}, SMOTE \cite{fernandez2018smote}) and proposes edge predictor to fuse augmented nodes into the original graph. GraphENS \cite{park2022graphens} discovers neighbor memorization phenomenon in imbalanced node classification, and generates minority nodes by synthesizing ego-networks according to similarity. GraphSHA \cite{li2023graphsha} only synthesizes harder training samples and proposes S{\small EMI}M{\small IXUP} to block message propagation from minority nodes to neighbor classes by generating connected edges from 1-hop subgraphs. Another branch of class-imbalanced learning on graphs is topology-aware logit adjustment. TAM \cite{song2022tam} adjusts margins node-wisely according to the extent of deviation from connectivity patterns to avoid inducing false positives of minority nodes.
Different from the methods mentioned above that focus on class imbalance, our paper attempts to address the issue of group imbalance. We emphasize that group imbalance is a significant source of unfairness in GNNs and, as a result, design group re-balancing strategies to enhance the fairness of graph learning.
\section{Conclusion}
\label{sec:conclusion}

In this paper, we provide a new perspective to address the unfairness in graph neural networks. We find that group distribution imbalance is a significant source of bias and simple re-balancing methods (\emph{e.g.}, re-weighting and re-sampling) can easily achieve competitive or even superior performance compared to existing state-of-the-art fair GNNs. To this end, we propose FairGB which consists of two modules to automatically balance the contributions of each group. Guided by theoretical analysis, we conduct linear interpolation between counterfactual node pairs to effectively mitigate bias. In order to further enhance fairness, we propose contribution alignment loss based on gradients and flexibly combine two modules. Experimental results demonstrate the effectiveness of our proposed FairGB in achieving state-of-the-art performance on three real-world datasets. Future research directions can delve more into understanding unfairness from a re-balancing perspective and better integration of the graph properties to mitigate bias.

\section*{Acknowledge}

This work was partially supported by the Research Grants Council of Hong Kong, No. 14202919 and No. 14205520.




\bibliographystyle{ACM-Reference-Format}
\bibliography{FairGB.bib}

\appendix
\end{document}